\documentclass{article}
\author{Alexandr Andoni \and Rina Panigrahy}
\date{Microsoft Research SVC}

\usepackage{fullpage}
\usepackage{amsthm}
\usepackage{amssymb}
\newtheorem{theorem}{Theorem}[section]
\newtheorem{lemma}[theorem]{Lemma}

\newtheorem{corollary}[theorem]{Corollary}
\newtheorem{remark}[theorem]{Remark}

\newtheorem{claim}[theorem]{Claim}

\usepackage{subfigure}

\usepackage{amsmath}
\usepackage{hyperref}
\usepackage{multirow}
\usepackage{tabularx}
\usepackage{colortbl}
\usepackage{color}
\usepackage{graphicx}
\usepackage[small]{caption}

\newcommand{\N}{{\mathbb{N}}}

\newcommand{\R}{{\mathbb{R}}}

\newcommand{\eps}{\epsilon}
\newcommand{\beq}{\begin{eqnarray}}
\newcommand{\eeq}{\end{eqnarray}}

\newcommand{\bt}{\tilde{b}}
\newcommand{\vb}{{\bar b}}
\newcommand{\sign}{\text{sign}}
\newcommand{\T}{T}

\newcommand{\E}[2]{{\mathbb{E}_{#1}\left[#2\right]}}

\DeclareMathOperator{\poly}{poly}

\DeclareMathOperator{\erfi}{erfi}
\DeclareMathOperator{\erf}{erf}
\DeclareMathOperator{\dx}{dx}
\DeclareMathOperator{\dy}{dy}
\newcommand{\pd}[2]{\frac{\partial#1}{\partial#2}}

\title{A Differential Equations Approach to Optimizing Regret Trade-offs}

\begin{document}
\maketitle
\pagestyle{plain}

\begin{abstract}
We consider the classical question of predicting binary sequences and
study the {\em optimal} algorithms for obtaining the best possible
regret and payoff functions for this problem. The question turns out
to be also equivalent to the problem of optimal trade-offs between
the regrets of two experts in an ``experts problem'', studied before by
\cite{kearns-regret}. While, say, a regret of $\Theta(\sqrt{T})$ is
known, we argue that it important to ask what is the provably optimal
algorithm for this problem --- both because it leads to natural
algorithms, as well as because regret is in fact often comparable in
magnitude to the final payoffs and hence is a non-negligible term.

In the basic setting, the result essentially follows from a classical
result of Cover from '65. Here instead, we focus
on another standard setting, of time-discounted payoffs, where the
final ``stopping time'' is not specified. We exhibit an explicit
characterization of the optimal regret for this setting.

To obtain our main result, we show that the optimal payoff functions
have to satisfy the Hermite differential equation, and hence are given
by the solutions to this equation. It turns out that characterization
of the payoff function is qualitatively different from the classical
(non-discounted) setting, and, namely, there's essentially a unique
optimal solution.
\end{abstract}


\section{Introduction}

Consider the following classical game of predicting a binary $\pm 1$
sequence. The player (predictor) sees a binary sequence $\{b_t\}_{t\ge
  1}$, one bit at a time, and attempts to predict the next bit $b_t$
from the past history $b_1,\ldots b_{t-1}$. The {\em payoff} (score)
of the algorithm is then the count of correct guesses minus the number
of the wrong guesses, formally defined as follows, for some target
time $T>0$, and where $\bt_t$ is the prediction at time $t$:
$$ 
A_T=\sum_{1\le t\le T} b_t\bt_t.
$$

One can view this game as an idealized ``stock prediction'' problem as
follows. Each day, the stock price goes up or down by precisely one
dollar, and the player bets on this event. If the bet is right, the
player wins one dollar, and otherwise she looses one dollar. Not
surprisingly, in general, it is impossible to guarantee a positive
payoff for all possible scenarios (sequences), even for randomized
algorithms. However, one could hope to give some guarantees when the
sequence has some additional property.

The above sequence prediction problem is in fact precisely equivalent
to the two experts problem (or multi-armed bandits problem), where one
considers two experts, via a reduction: one side of the reduction
follows simply by using two experts, one always predicts ``+1'' and
another always predicts ``-1''. Then one measures the {\em regret} of
an algorithm: how much worse one's algorithm does as opposed to the
best of the two experts (in hindsight, after seeing the sequence),
which is equal to $|\sum_{1\le t\le T} b_t|$. We will henceforth will
refer to $\sum_{1\le t\le T} b_t$ as the ``height'' of the sequence
(as in the height of a growth chart of a stock). Regret has been studied
in a number of papers, including \cite{cover-binary,weighted-majority,cover-portfolios,auer-nonstoch,audibert-bubeck}. A classical result says that one can obtain a
regret of $\Theta(\sqrt{T})$ for a sequence of length $T$, via, say,
the weighted majority algorithm of \cite{weighted-majority}.  Note that
the payoff per time step $b_t\bt_t$ is essentially equivalent to the
well known absolute loss function $|b_t-\bt_t|$ (see for example
~\cite{plg}, chapter 8)\footnote{since when $|b_t| = 1$, $|b_t-\bt_t|
  = |b_t||b_t-\bt_t| = |1 - b_t \bt_t| = 1 - b_t \bt_t$.  Thus the
  absolute loss function is the negative of our payoff in one step
  plus a shift of $1$. Also $b_t$ values from $\{-1,1\}$ or $\{0,1\}$
  are equivalent by a simple scaling and shifting transform.}.

Obtaining a regret of $\Theta(\sqrt{T})$ has since become the golden
standard for many similar expert learning problem. But is this the
best possible guarantee? While there is a lower bound of
$\Omega(\sqrt{T})$, it is natural to ask what is the {\em
  optimal} algorithm for minimizing the regret, departing from
asymptotic notation. Note that the weighted majority algorithm may not
be optimal, even if it obtains the ``right order of magnitude''. More
generally, one can ask what exactly are all possible payoff functions
one can achieve as a function of the total payoffs of the two arms
(height, in our case).

In this paper, we undertake precisely this task, of studying the
algorithms that obtain the {\em optimal, minimal} regret possible and
characterize the possible payoff functions.  Our results also lead to
optimal regret trade-offs between two experts in the experts problem
from the equivalence between the two problems.  The latter problem has
been previously studied by \cite{kearns-regret}, and
later by \cite{KP11}, to address, say, an investment scenario where
there may be two experts one risk taking and another conservative and
one may be willing to take different regrets with respect to these two
experts.  In particular, it is known that it is possible to get regret
$O(\sqrt {T\log T})$ with respect to one expert and $1/T^{\Omega(1)}$
with respect to the other.


There are several reasons to study such optimal algorithms and compute
the exact trade-off curves. First of all, such an optimal algorithm
may be viewed as more ``natural'', for example, because if an
autonomous system has the same optimization criteria (of minimizing
regret), it would arrive at such an ``optimal'' solution. Second, it
is worthwhile to go beyond the asymptotics of a $\Theta(\sqrt{T})$
bound. Specifically, often the final value of a sequence is actually
of the order of $\sqrt{T}$, such as for a random sequence. Although,
we do not expect to obtain a positive payoff for a random sequence, a
large fraction of all sequences still have $O(\sqrt{T})$ value. In
such a scenario, it is critical to obtain the best possible constant
in front of the $\sqrt{T}$ regret bound. When the value of the
sequence is indeed around $\Theta(\sqrt{T})$, an algorithm with a
regret of $\Theta(\sqrt{T})$ achieves a constant factor approximation,
and improving the leading constants leads to an approximation factor
which is a better constant. For example, in several investment
scenarios it is known that the payoffs of the experts (or stocks) in
time $T$ is barely more than $O(\sqrt T)$ (see, for example, the Hurst
coefficient measurements of financial markets in \cite{BPS-sp500,S-fractional}).  In such settings, the precise constant in regret
term can translate into a difference between gain and loss.  Indeed,
we find that our algorithm can have a regret that is about $10\%$
lower than that of the well known weighted majority algorithm and, at
several positions on the curve, our payoff is improved by as much as
$0.3 \sqrt T$ (see figure \ref{fig:bettingStrategies}).  We also
obtain the exact trade-off curve between the regrets with respect to
two arms (see figure \ref{fig:tradeoff}).

We note that, in the vanilla setting, when there is a time bound $T$,
the solution already follows from the results of \cite{cover-binary}
(see also \cite{cesa1997use,plg}), who gave a characterization of all
possible payoffs back in 1965. One can also obtain the optimal
algorithm by computing a certain dynamic programming, similar to an
approach from \cite{Kohn-diff}. Yet, the resulting algorithm has a
betting strategy and payoff function that are {\em time-dependent} as
well as depend on the final stopping time $T$. These dependencies
introduce issues and parameters that are hard to control in reality
(often the predictor does not really know when the time ``stops''). To
understand the time-independent strategies, we are led to consider the
another classic setting of {\em time-discounted} payoffs
(see \cite{gittins,tsitsiklis}).

Thus we focus our study on regret-optimal algorithms in the
time-discounted setting, where payoff is discounted, and there is no
apriori time bound.
Formally, we define a {\em $\rho-$discounted} version of payoff at
some moment of time $T$, for a discount factor $\rho\in (0,1)$, as
$$ A_T^\rho=\sum_{t \ge 0} b_{T-t}\bt_{T-t} \cdot \rho^{t}
$$ The question then is to minimize the regret with respect to this
quantity, as a function of (discounted) height.  One can also see this
scenario as capturing the situation where we care about a certain
``attention'' window of time (given by $\rho$).  One of the
consequences of our study is that, when the strategies are
time-independent, the characterization of the optimal
regret/algorithms becomes quite different.



\subsection{Statement of Results}

In general, we study the optimal regret {\em curves}. Namely, we
measure the payoff and regret as a function of the ``height'' of the
sequence (the sum of the bits of the sequence, as defined above; one
can also take a discounted sum).  Note that comparing against height
amounts to comparing the performance of our algorithm against that of
two static experts: one that always predicts +1 (``long the stock''), and
another that always predicts -1 (``short the stock''). The former obtains
a payoff equal to the height and the latter obtains a payoff equal to
negative height.

We use the notion of a {\em payoff function} --- a real function
$f$, which assigns algorithm's payoff $f(x)$ for each height value
$x$. In particular, for fixed algorithm and a height $x$, let $f_T(x)$
denote the minimum payoff over all sequences with height $x$ at time
$T$. For a certain function $f_T$, we will say that $f_T(x)$ is
feasible if there is an algorithm with payoff at least $f_T(x)$
over all possible sequences $\{b_t\}$ such that $h(\{b_t\}) = x$. In
the discounted scenario, the notion of height becomes the {\em
  discounted height}: $h_T^\rho(\{b_t\}_{t\le T})=\sum_{t\ge0}
b_{T-t} \rho^{t}$. More importantly, for time-independent strategies (in
the discounted setting) we will say that $f(x)$ is feasible if the
payoff is at least $f(x)$ for (discounted) height $x$ at {\em all
  times} (feasible in steady state).


Our goal will be to optimize the regret, defined for a payoff
function $f$, as follows:
$$
R(f)=\max_{x}|x|-f(x),
$$
where $x$ ranges over all possible (discounted) heights.


Note that $|x|$ is the maximum of the payoff of the two constant experts.
In general, we allow bets $\bt_t$ to be bounded reals in the interval
$[-1,1]$. In such a case, it is sufficient to consider deterministic
strategies only. One can also consider the version of the problem when
there is no restriction on the range of values for $\bt_t$. We will
refer to this case as the {\em sequence prediction problem with
  unbounded bets}. This will be useful in deriving bounds for the
standard case with bounded bets.

For starters, we remind the result for the vanilla, non-discounted,
fixed stopping time setting, which follows from \cite{cover-binary},
and is related to Rademacher complexity of the predictions of the two
experts (see \cite{cesa1997use,plg}). The theorem below also extends to the
discounted scenario, with fixed stopping time $T$. See Appendix
\ref{apx:discountedFixed} for discussion of this settings.

\begin{theorem} 
\label{thm:vanillaRegret}
Consider the problem of prediction of binary sequence. The minimal
possible regret is
$$ R=\min_f R(f)= \sqrt{\frac{2}{\pi}} \sqrt{T}+O(1).
$$ There is a prediction algorithm (betting strategy) achieving this
optimal regret and has $f(x) = |x|-R$. The actual corresponding
betting strategy may be computing via dynamic programming.

Furthermore, $f$ is feasible iff $\sum f(x) p(x)=0$ where $p(x)$ is
the probability of a random walk of length $T$ to end at $x$ (i.e.,
$\E{}{f(x)}=0$ for $x$ being the height of a random sequence). For
bounded bet value, we have the additional constraint that $f$ is
$1$-Lipschitz.
\end{theorem}

\paragraph{Time-independent strategies.}
Our main result is for optimal regret curves in the setting of discounted and
time-independent strategies.
We characterize the set of all-time feasible
$f$'s. For this, we define a certain ``optimal'' curve function, which
will be central to our claims. For constants $c_1,c_2$, define the
following function:

$$ F_{c_1,c_2} (x)=c_1 \left(x\cdot\erfi(x)-e^{x^2}/\sqrt{\pi}\right) + c_2 x,$$ 

where
$\erfi(x)=i\cdot \erf(ix)$ is the imaginary error function. We also
define $\hat F_{c_1,c_2}$ to be the function obtained by bounding the
derivative of $F$ to lie in $[-1,1]$. That is $\hat F = F$ when $|F'|
\le 1$ and $\hat F' = \sign(F')$ when $|F'| > 1$.

\begin{theorem}[Main]
\label{thm:ffeasible}
Consider the problem of discounted prediction of binary sequence with
the discount factor of $\rho=1-1/n$ (corresponding to a ``window
size'' of $n$). A payoff function $f$ is feasible in the steady state if
there exist constants $c_1, c_2$ such that for all $x\in[-n,n]$:
$$f(x)\le 
\sqrt{n}\cdot \hat F_{c_1,c_2}(x/\sqrt{n})-O(1).$$

Conversely, if there exists a function $g$ such that $f(x)=\sqrt
n\cdot g(x/\sqrt n)$ for infinitely many $n$ and $g$ is piecewise
analytic\footnote{In fact, it suffices to assume that the first three
  derivatives of $g$ exist instead of requiring it to be analytic.}
then $g(x) \le \hat F_{c_1,c_2}(x)$ for some constants $c_1, c_2$.


Hence, the minimum $\rho$-discounted regret is, for
$C=\min_{\alpha\ge1} \tfrac{1}{\sqrt{\pi}}\cdot
\tfrac{\alpha}{\erfi(\sqrt{\ln\alpha})}$:
$$
\min_{f_\rho} R(f_{\rho}) = C \sqrt{n} + O(1).
$$
\end{theorem}

We note that the above characterization follows from a ``limit view''
of the corresponding dynamic programming characterizing the payoff
function, which leads to a differential equation formulation of the
question. Such an approach has been previously undertaken by
\cite{Kohn-diff} to show that many differential equations can be
realized as two-person games, as is also the case in our scenario.

In particular, to prove Theorem~\ref{thm:ffeasible}, we show that $f$
needs to satisfy the inequality \beq\label{recinequality} f(x) \ge
\frac{f(\rho x+1) + f(\rho x-1)}{2\rho}. \eeq

It turns out that, after the correct rescaling, and taking the process
to the limit, we obtain a differential equation. Namely, let $g(x) =
f(\sqrt n x)/\sqrt n$ denote a normalized version of $f$ where the
axes are scaled down by a factor or $\sqrt n$ (the standard deviation
of the height). We will assume that $g$ is (piece-wise) analytic\footnote{In
fact all we will need is that it is twice differentiable.}. Then, as
$n$ approaches infinity, the above inequality implies the following
differential inequality:
$$
g'' - 2xg' + 2g \le 0
$$ If we replace the inequality by equality, we obtain the Hermite
differential equation which has as its solution the aforementioned
functions $F_{c_1,c_2}$. While our solutions are close to these
differential equation solutions $F_{c_1,c_2}$, we also point out the
curious fact that if we insist on the constraint \eqref{recinequality}
being an equality, then the only solution is $f(x)=0$. Thus the
relaxation into an inequality seems necessary to capture the feasible
set of functions.

The algorithm from the above theorem is explicitly given. In
particular, it computes the current discounted height $x$, and then
outputs the bet $\tilde b(x)=\frac{f(\rho x+1) - f(\rho x-1)}{2}$ for the next time step, for $f$
from Theorem \ref{thm:ffeasible}. Surprisingly, the
characterization of the feasible payoff functions $f$ is very
different when the strategies are time-independent as opposed to the
time-{\em dependent} case.  In particular, in the time-independent
case, there are only two degrees of freedom as compared to the
time-dependent case when there were infinite (or $\approx n$) degrees
of freedom.

See figure \ref{fig:bettingStrategies} for the plots of the resulting
betting strategy as compared to the one resulting from the
multiplicative weights update algorithm (which also happens to be a
time-independent strategy). Also, see figure \ref{fig:curves} for the
resulting payoff function $f$ (where the axes have been scaled down by
$\sqrt{n}$). After scaling $x$ down by $\sqrt{n}$, we obtain that
$\tilde b(x)$ tends to $F'(x)$ as $n\rightarrow \infty$.

\begin{figure*}[ht]
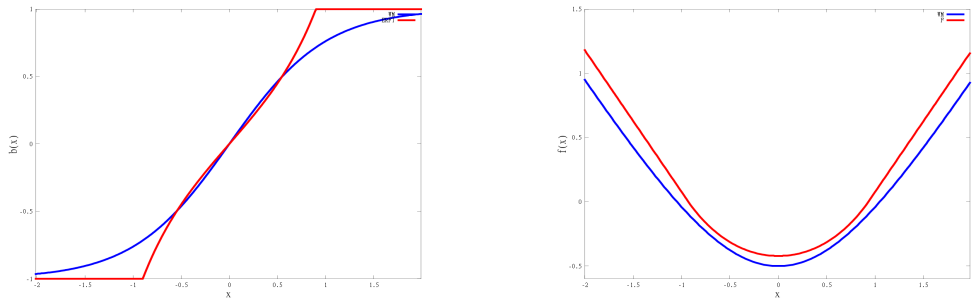

\begin{center}
\mbox{
\subfigure[\label{fig:bettingStrategies}
Scaled graphs of the betting strategies for $n=100$ for the weighted majority betting strategy $\bt(x) = \tanh(x)$ (blue) and the betting strategy resulting from Theorem~\ref{thm:ffeasible}, which is equal to $c\erfi(x)$ capped at $\pm 1$ (red). ($x$ axis has been scaled down by $\sqrt{n}$.)
]
{
\includegraphics[width=66mm, height=44mm]{betting.pdf}
}
\quad
\subfigure[\label{fig:curves}  Scaled graphs for the payoff curves $f(x)$, for $n=100$, for the weighted majority algorithm (blue) and the solution resulting from Theorem~\ref{thm:ffeasible} (red). ($x$ and $y$ axes have been scaled down by $\sqrt{n}$.)]{
\includegraphics[width=66mm, height=44mm]{curves.pdf}
}
}
\caption{Graphs for the prediction of binary sequences problem.}
\end{center}
\end{figure*}

\paragraph{Trading off regrets between two experts.}
We also relate our problem to experts problem with two experts (or the
multi-armed bandit problem in the full information model with two
arms/experts). Here, in each round, each expert has a payoff in the
range $[0,1]$ that is unknown to the algorithm. For two experts, let
$b_{1,t}, b_{2,t}$ denote the payoffs of the two experts at time
$t$. The algorithm pulls each arm (expert) with probability
$\bt_{1,t}, \bt_{2,t} \in [0,1]$ respectively where $\bt_{1,t} +
\bt_{2,t} = 1$.  The payoff of the algorithm in this setting is $A'_T
:= \sum_{t=1}^T b_{1,t}\bt_{1,t} + b_{2,t}\bt_{2,t} $. The objective
of the algorithm is to obtain low regret with respect to the two
experts. We note that this was first studied in ~\cite{kearns-regret}.

We achieve the optimal tradeoff between the regrets with respect to
the two experts by reducing it to an instance of the sequence
prediction problem. In particular, define the loss of a payoff
function $f$ as the negative of the minimum value of $f$. Then we show
that the regret/loss trade-off for the sequence prediction problem is
tightly connected to the trade-off of the regrets with respect to the
two experts.  Hence, we also derive the regret trade-off for the case
of two experts. (See figure \ref{fig:tradeoff} for the trade-off curve
for the regrets in the two experts problem.)

\begin{figure*}[ht]
\begin{center}
{
\includegraphics[width=88mm, height=55mm]{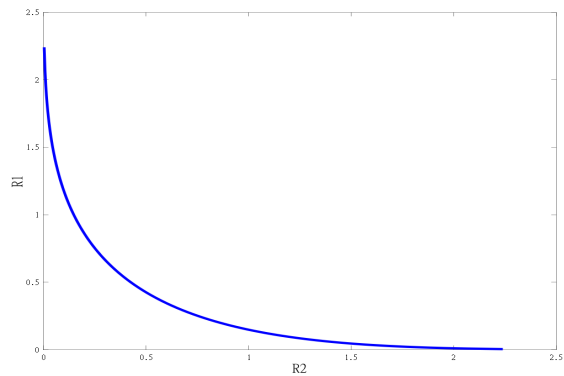}
\caption{
\label{fig:tradeoff}
Tradeoff between two regrets $R_1, R_2$ (scaled down by $\sqrt{n}$)
for the two experts problem (time-discounted case).  } }
\end{center}
\end{figure*}

\begin{theorem}\label{twoexperts}
Consider the problem of trading off regrets $R_1, R_2$ with respect to
two experts.  Regrets $R_1, R_2$ are achievable in the time-discounted
setting if and only if there exists an $\alpha > 0$ such that
$\T(\alpha R_1/\sqrt n) + \T(\alpha R_2 /\sqrt n) \ge \alpha /\sqrt
\pi$, where $\T(x) = \erfi(\sqrt{\ln x})$ for $x \ge 1$.
\end{theorem}

\paragraph{Multiple scales.}
Finally, we investigate the possible payoff functions at multiple time
scales $\rho$ (window sizes). Several earlier papers
considered regrets at different time scales; see \cite{blum-mansour,freund-schapire-singer-warmuth,hazan-seshadhri,vovk,KP11}.  We
consider two different time scales, $\rho_1 = 1 - 1/n_1$ and $\rho_2 =
1 - 1/n_2$, although a similar result can be obtained for a larger
number of time scales.

We exhibit the necessary and sufficient condition for a feasible
payoff function, as window size goes to infinity. In particular,
suppose that $n_1 = a_1 n$ and $n_2 = a_2 n$ where $n$ tends to
infinity. Let $x_1$ and $x_2$ be the time discounted heights for the
two different time scales. We can ask if it is possible to get (time
discounted) payoff functions $f_1(x_1, x_2)$ and $f_2(x_1, x_2)$ at
time scales $n_1$ and $n_2$ respectively. Again we apply the
coordinate rescaling by $\sqrt{n}$ for both $x_1,f_1$ and $x_2,f_2$.
\begin{theorem}
\label{thm:multitime}
For $n\ge 1$, fix two windows $n_1=a_1n$ and $n_2=a_2n$. As $n$ goes
to infinity, there is are payoff function
$f_1(x_1,x_2)=\sqrt{n}g_1(x_1/\sqrt{n},x_2/\sqrt{n})$ for the discount
rate $\rho_1=1-1/n_1$ and
$f_2(x_1,x_2)=\sqrt{n}g_2(x_1/\sqrt{n},x_2/\sqrt{n})$ for the discount
rate $\rho_2=1-1/n_2$, as $n$ goes to infinity, if and only if the
following system of partial differential inequalities is satisfied:
\\
$
E_1 \triangleq -\frac{1}{2}\left(\pd{}{x_1}+\pd{}{x_2}\right)^2 g_1 + \left(a_1^2x_1\pd{}{x_1}+a_2^2x_2\pd{}{x_2}\right) g_1 - a_1^2 \cdot g_1 \ge 0
$
\\
$
E_2 \triangleq -\frac{1}{2}\left(\pd{}{x_1}+\pd{}{x_2}\right)^2 g_2 + \left(a_1^2x_1\pd{}{x_1}+a_2^2x_2\pd{}{x_2}\right) g_2 - a_2^2 \cdot g_2 \ge 0
$
\\ 
$
E_1 + E_2 \ge \left|\left(\pd{}{x_1}+\pd{}{x_2}\right) (g_1 - g_2)\right|.
$
\end{theorem}

We do not seem to have explicit analytical solution for the above
system of inequations, and so perhaps one would have to rely on
numerical simulations to solve it. This part is deferred to Appendix
\ref{apx:multiScale} due to space limitation.

\subsection{Related Work}
\label{sec:relatedWork}

There is a large body on work on regret style analysis for
prediction. Numerous works including ~\cite{cover-binary,cesa1997use}
have examined the optimal amount of regret achievable with respect to
multiple experts. Many of the results in this body of work can be
found in ~\cite{plg}. It is well known that in the case of static
experts, the optimal regret is exactly equal to the Rademacher
complexity of the predictions of the experts (chapter 8 in
~\cite{plg}).  Recent works, including
~\cite{abernethy2006continuous,abernethy2008optimal,mukherjee2008learning}, have extended this analysis to other
settings.  Measures other than the standard regret measure have been
studied in \cite{rakhlin2010online}. 
Also related is the NormalHedge
algorithm \cite{normalhedge}, though it differs in both the setting
and the precise algorithm. Namely, NormalHedge looks at undiscounted
payoffs and obtains strong regret guarantees to the epsilon-quantile
of best experts. We look at two experts case (where epsilon-quantile
is not applicable) and seek to obtain provably optimal regret.

Algorithms with performance guarantees within each interval have been
studied in \cite{blum-mansour,freund-schapire-singer-warmuth,vovk}
and, more recently, in \cite{hazan-seshadhri,KP11}.  The question of
what can be achieved if one would like to have a significantly better
guarantee with respect to a fixed arm or a distribution on arms was
asked before in \cite{kearns-regret,KP11}. Tradeoffs between regret
and loss were also examined in \cite{vovk-game}, where the author
studied the set of values of $a, b$ for which an algorithm can have
payoff $a OPT+b\log N$, where $OPT$ is the payoff of the best arm and
$a, b$ are constants.  The problem of bit prediction was also
considered in \cite{freund-coin}, where several loss functions are
considered. Numerous papers
(\cite{blum1997empirical,helmbold1998line,agarwal2006algorithms}) have
implemented algorithms inspired from regret syle analysis and applied
it on financial and other types of data.

\section{Time-Independent Prediction Algorithms}
\label{sec:timeIndep}

In this section we study the optimal regret and algorithms for the
time-independent strategies and regret curves. We consider the
time-discounted setting, thereby proving
Theorem~\ref{thm:ffeasible}.

As mentioned in the introduction, we consider a payoff $f$ to be
feasible if there is a prediction algorithm that achieves a payoff of
at least $f(x)$ for the discounted height $x$ at all times $t\ge
1$.  We will argue that, without loss of generality, we can assume
that the betting strategy $\tilde b(x)$ is time independent and the
payoff always dominates the function $f(x)$.

Observe that for a time independent betting function, the payoff
function it achieves is also time independent in the limit.

\begin{claim}
\label{clm:timeIndep}
If $f$ is feasible (in steady state), then there is a time-independent
betting strategy $\tilde b$ that achieves payoff function $f$.
\end{claim}
\begin{proof}[Proof of Claim \ref{clm:timeIndep}]
Remember that we use discount factor of $\rho=1-1/n$, where $n\ge 1$
is the ``window'' size.  Assume there is a time-{\em dependent}
betting strategy $\bt_t(x)$ that achieves payoff at least $f(x)$ in
the steady state.  We consider the average of these betting strategies
over a long interval and argue that it changes only slightly over
time. Note that the time shifted strategy $\bt_{t-i}$ also achieves
payoff at least $f(x)$ at all times. This means that an average of a
large number of such shifted betting strategies also achieves this.
Consider the average strategy $\mu_t (x) = \frac{1}{N}\sum_{i=1}^N
\bt_{t-i} (x)$, and note it is essentially constant over a small
window for a sufficienly large $N$. For example, if we choose $N >
\exp(n)$ the differences in $\mu_t$ over a window of size $\poly(n)$
are exponentially small.  Since we are time discounting at rate
$1-1/n$, it suffices to ignore anything outside such a window of size
$\poly(n)$.
\end{proof}

We will characterize the payoff functions that are feasible for
time-independent betting strategies.

\begin{lemma}\label{lem:inequality}
If there is a time-independent betting strategy with payoff function
$f(x)$ then
\begin{equation}\label{inequality}
f(x) \ge \frac{f(\rho x+1) + f(\rho x-1)}{2\rho}.
\end{equation}
Conversely if $f$ satisfies the above inequality and $f(0) \le 0$, then
it is feasible with unbounded bets. In particular the betting strategy
$\tilde b(x) = \frac{f(\rho x+1) - f(\rho x-1)}{2}$ achieves a payoff
function at least $f$. For the bounded bets case, we need the
additional constraint that $\tilde b(x)$ computed thus satisfies
$|\tilde b(x)| \le 1$
\end{lemma}

\begin{proof}
Note that since the payoff at time $t$ is $\tilde b(x_t)b$ (where $b
= b_t$), we have $\rho f(x) + b\tilde b(x) \ge f(\rho x+b)$ where $b
\in \{\pm 1\}$.  This is because at time $t-1$ there is some sequence
of height $x$ with payoff $f(x)$.  Thus, we have $\rho f(x) +
\tilde b(x) \ge f(\rho x+1)$ and, similarly, $\rho f(x) - \tilde b(x) \ge
f(\rho x-1)$. Averaging the two we get inequality \eqref{inequality}.

To prove the converse we can use induction on time $t$ to show that
the stated betting strategy achieves payoff at least $f(x)$. Clearly
at $t=0$, $x=0$ and since $f(0) \le 0$ the condition is satisfied.
Further, if the height is $x$ at time $t-1$ then at the next step the
payoff is at least $\rho f(x) + b\tilde b(x) \ge f(\rho x+b)$ for
$b\in \{-1,1\}$ which follows from the inequality \eqref{inequality}.
\end{proof}

We now proceed to proving the main claims of Theorem
\ref{thm:ffeasible}. In particular, we start by showing the ``converse''
direction. For this, we will show that, in the limit, the payoff
function has to satisfy a certain differential equation, when property
scaled. The next lemma proves precisely this switch.

\begin{lemma}
Let $g(x) = f(\sqrt n x)/\sqrt n$, and assume it is piece-wise
analytic. Then as $n \rightarrow \infty$, condition ~\eqref{inequality}
becomes
\begin{equation}\label{diffinequality}
g'' - 2xg' + 2g \le 0.
\end{equation}
\end{lemma}

\begin{proof}
Rescaling and setting $\delta = 1/\sqrt n$ (i.e., $\rho=1-\delta^2$) in inequality
\eqref{inequality} gives us:
$$(1-\delta^2)g(x) \ge \frac{g((1- \delta^2)x+\delta) +
  g((1-\delta^2)x-\delta)}{2}.$$

Using Taylor expansion on $g$, we obtain
\beq \nonumber (1-\delta^2)g(x) \ge
g(x) - \delta^2 x g'(x) + (1/2)\delta^2(1+\delta^2 x^2)g''(x) +
O(\delta^3) g'''(x - \delta^2x\pm\delta) \\ \nonumber 0 \ge g(x) - x
g'(x) + (1/2)(1+\delta^2 x^2)g''(x) + O(\delta) g'''(x -
\delta^2x\pm\delta) 
\eeq

As $\delta=1/\sqrt{n} \rightarrow 0$, we obtain that $g'' - 2xg' +
2g \le 0$.
\end{proof}

We note that if we replace the inequality \eqref{inequality} with the
equality, we obtain the Hermite differential equation: 
\beq
\label{eqn:Hermite}
g'' - 2xg' + 2g = 0. \eeq 
Differential equation \eqref{eqn:Hermite} has a general
solution of the form $F_{c_1,c_2} =c_1 (x\erfi(x)-e^{x^2}/\sqrt{\pi})
+ c_2 x$, where $\erfi(x)=i\cdot\erf(ix)$ is the imaginary error function.

\begin{remark}
Note that, for example, this ``limiting payoff function''
$F_{c_1,c_2}$ satisfies the ``limiting'' $T\rightarrow \infty$
characterization similar to Theorem~\ref{thm:discountedFixed}. Namely,
for any constants $c_1,c_2$, we have that $\int p_\rho(x)
F_{c_1,c_2}(x)\dx=0$, where $p_\rho(x)$ is the distribution of the
$\rho$-decayed random walk to be at height $\sqrt{n}x$, at the limit
of $n,T\rightarrow \infty$. (Note that $p_\rho$ converges to $N(0,1)$
when $n\rightarrow \infty$.)
\end{remark}

In the following, we show that the solutions for the steady-state
payoffs are essentially characterized by functions $F_{c_1,c_2}$.
Note that we thus obtain solutions that have only two degrees of
freedom. This is in stark contrast to the time-dependent strategies,
where there is an infinite number of degrees of freedom (see Appendix
\ref{apx:fixedTime}).

The next lemma shows that if $g(x)=f(\sqrt{n} x)/\sqrt{n}$ satisfies
the differential {\em inequality}, then $g$ must be dominated by
$F_{c_1, c_2}$, i.e., a solution to the Hermite differential equation.

\begin{lemma}
Suppose $g$ satisfies $g'' - 2xg' + 2g \le 0$. Then there exist some $c_1,c_2$, such that $g\le F_{c_1,c_2}$.
\end{lemma}
\begin{proof}
There is a unique solution $y=F_{c_1,c_2}$ such that $y(0) = g(0)$ and $y'(0) =
g'(0)$. Now look at $h = g-y$. We will show that $h \le 0$. Observe
that $h$ satisfies $h'' - 2xh' + 2h \le 0$ and $h(0) = h'(0) = 0$.

We will make the substitution $u = xh' - h$. Hence we have that $u' = xh''$, and thus
$u'/x - 2u \le 0$. or $u' \le 2ux $ for $x\ge 0$ and $u' \ge 2ux$ for
$x \le 0$ and $u(0) = 0$. This implies that $u \le 0$. This means that
$xh'-h\le 0$ which implies $h \le 0$.
\end{proof}

So far we have ignored the condition that the $|b| \le 1$ thus
allowing unbounded bets. In the following claim, we consider the case
of bounded bets and show that in this case the function $g$ has a
bounded derivative.

\begin{claim}
With bounded bets $|b| \le 1$, the function $g$ must also satisfy the
constraint $|g'(x)| \le 1$ as $n \rightarrow \infty$.
\end{claim}
\begin{proof}
For $\delta=1/\sqrt{n}$, we have that
\beq 
\nonumber
(1-\delta^2) g(x) + \delta \ge g((1-\delta^2)x+\delta) 
\implies -\delta
g(x) + 1 \ge \frac{g((1-\delta^2)x+\delta) - g(x)}{\delta} 
\eeq

Considering $\delta \rightarrow 0$ gives $g'(x) \le 1$. Similarly we
get $-g'(x) \le 1$.
\end{proof}

Suppose we choose a solution $g = F_{c_1,c_2}$, this would correspond to
the betting strategy $b(x) = c_1\cdot \erfi(x) + c_2$. Note that $F$
doesn't satisfy $|g'(x)| \le 1$, but a simple capping of its growth
when $|F'| \ge 1$ gives a alternate function $\hat F$ (see figure
\ref{fig:curves}) that satisfies the extra condition. This essentially
corresponds to capping $b(x)$ so that $|b(x)| \le 1$. Let $\hat b(x)$
denote the capped version of $b(x)$ that can be used for bounded bets.

This concludes the ``converse'' part of Theorem~\ref{thm:ffeasible}.
Next, we switch to showing the forward direction, that if (a properly
scaled) $f$ is dominated by $F$, then it is also a valid payoff
function. In particular, in the next lemma, we show that the solutions
to the differential inequality can be made to satisfy the original
recursive inequality~\eqref{inequality} with a small error term.

\begin{lemma}
\label{lem:diffeqnerror}
For any constants $c_1, c_2$, for the bounded bets case, there is a
function $\hat g(x) =\hat F(x) - O(1/\sqrt{n})$ such that $\sqrt n \cdot \hat
g(\sqrt n x)$ satisfies the inequality \eqref{inequality}.

With unbounded bets, there is a function $g(x) = F(x)\cdot
e^{-O(x^2/n+1/n)}$ such that $\sqrt n g(\sqrt{n} x)$ satisfies the
inequality \eqref{inequality}.
\end{lemma}

\begin{proof}
Let $\delta=1/\sqrt{n}$.  We will argue that the $O(\delta)$ slack is
sufficient to account for the error in the Taylor approximation in the
bounded bets case. To see this note that the error in the Taylor
approximation is $\delta^2 x^2 g''(x) + O(\delta) \bar
g'''((1-\delta^2)x \pm \delta)$, where $\bar g'''(x \pm \eps)$ denotes
the average of $g'''$ at two points in the range $x \pm \eps$. We will
look at the interval where $|F'(x)| \le 1$.  For constant $c_1,c_2$
the end points of this interval are also constants which implies all
the terms in the error expression are constants (since $f$ is
independent of $\delta$). Thus the error is at most $O(\delta)$. So it
suffices to satisfy the condition $g'' - 2xg' + 2g < -O(\delta)$ which
is satisfied by $F - O(\delta)$. For the region where $|F'(x)| \ge 1$,
note that in $\hat F$ we are capping $|\hat F'(x)| = 1$ and since
$\hat g$ is $\hat F$ shifted down, it also satisfies the inequality.

For the case of unbounded bets, observe that the recursive
inequality holds if we satisfy the following, per the approximation of the Taylor
series:
$$
\bar g''(x -  x^2\delta^2 \pm \delta)(1+ x^2 \delta^2) - 2xg'(x) +2 g(x) \le 0.
$$

For simplicity of explanation, consider $x \ge 0$. Note that $\delta x
\le 1$.  We have $F''(x) = c_1 e^{x^2}$.  Suppose we look for a function $g$ that satisfies: $g''(x)$ is even and is
increasing in $x$ when $x\ge 0$ and $g''(x \pm \epsilon) \le
g''(x)/e^{4(x^2 \epsilon^2 + \epsilon^2)}$ for a big enough constant
in the $O$ --- we will later verify that our resulting $g$ indeed satisfies this (note that this is satisfied for $F$). Then, since $(1+ x^2
\delta^2) \le e^{O(x^2\delta^2)}$, the above inequality is satisfied as along as
$$
g''(x) \le e^{-O(x^2\delta^2  + \delta^2)}2(xg'(x) - g(x))
$$

Again, for $u = xg'-g$, we get: $\tfrac{u'}{x} \le  e^{-O(x^2\delta^2  + \delta^2)} 2u$ which holds if 
$u \le e^{x^2-O(\delta^2 + \delta^2)}$. 

So it suffices that $xg' -g = e^{x^2-O(x^2\delta^2 +
  \delta^2)}$. Dividing by $x^2$, we get: $(g/x)' =
e^{x^2-O(x^2\delta^2 + \delta^2)}/x^2$. Note that without the
correction terms the earlier differential equation $(g/x)' =
e^{x^2}/x^2$ has the solution $g = F$, and so for the new equation
there is a solution $g = F e^{-O(x^2\delta^2 + \delta^2)}$. Note that
this $g$ also satisfies $g''(x \pm \epsilon) \le g''(x)/e^{4(x^2
  \epsilon^2 + \epsilon^2)}$ that we had assumed.
\end{proof}

Our Theorem~\ref{thm:ffeasible} is hence concluded by
Lemma~\ref{lem:inequality} and Lemma~\ref{lem:diffeqnerror}. In the
following we remark that obtaining a (non-trivial) solution that
preserves the equation \eqref{eqn:Hermite} precisely is impossible.

\begin{remark}
If we convert the condition \eqref{inequality} into an equality then
the only satisfying analytic solution $f$ is $f(x) = cx$ for a
constant $c$. Thus the relaxation into an inequality seems to be
necessary to find all the feasible payoff functions
\end{remark}
\begin{proof}
If we require the equality $f(x) = \frac{f(\rho x+1) + f(\rho
  x-1)}{2\rho}$ then applying this recursively $i$ times gives:
$$
f(x) = \rho^{-i} \frac{\sum_{b_1,\cdots b_i \in \{-1,1\}} f(\rho^i x + \rho^{i-1} b_1 + \rho^{i-2}\cdots + b_i)}{2^i}
$$

We apply Taylor series to $f(\rho^ix+\rho^{i-1} b_1 + \cdots +
b_i)$ around the point $y=\rho^{i-1} b_1 + \cdots + b_i$ to conclude
that
$f(y+\rho^ix)=f(y)+\rho^ixf'(y)+\rho^{2i}x^2f''(y\pm\rho^ix)$. 
We now consider the following difference
$$
f(x)-f(0) = \rho^{-i} \frac{\sum_{b_1,\cdots b_i \in \{-1,1\}} f(\rho^ix+\rho^{i-1} b_1 + \cdots + b_i) - f(\rho^{i-1} b_1 + \cdots + b_i)}{2^i}
$$ and using the above expansion, we have
$$
f(x)-f(0) = \frac{\sum_{b_1,\cdots b_i \in \{-1,1\},y=\rho^{i-1} b_1 + \cdots + b_i} xf'(y)+\rho^ix^2f''(y\pm \rho^ix)}{2^i}.
$$

Taking $i$ tend to infinity, we conclude that $f(x)-f(0)=x\cdot \int
p_\rho(y)f(y)\dy$. This implies that $f$ is of the form
$f(x)=cx+a$. Moreover substituting $f$ into the equality condition, we
obtain that $a=0$.
\end{proof}



\bibliographystyle{splncs03}
\bibliography{regret}

\begin{thebibliography}{10}
\providecommand{\url}[1]{\texttt{#1}}
\providecommand{\urlprefix}{URL }

\bibitem{abernethy2006continuous}
Abernethy, J., Langford, J., Warmuth, M.: Continuous experts and the binning
  algorithm. Learning Theory pp. 544--558 (2006)

\bibitem{abernethy2008optimal}
Abernethy, J., Warmuth, M., Yellin, J.: Optimal strategies from random walks.
  In: Proceedings of The 21st Annual Conference on Learning Theory. pp.
  437--446. Citeseer (2008)

\bibitem{agarwal2006algorithms}
Agarwal, A., Hazan, E., Kale, S., Schapire, R.: Algorithms for portfolio
  management based on the newton method. In: Proceedings of the 23rd
  international conference on Machine learning. pp. 9--16. ACM (2006)

\bibitem{audibert-bubeck}
Audibert, J.Y., Bubeck, S.: Minimax policies for adversarial and stochastic
  bandits. COLT  (2009)

\bibitem{auer-nonstoch}
Auer, P., Cesa-Bianchi, N., Freund, Y., Schapire, R.: The nonstochastic
  multi-armed bandit problem. SIAM J. Comput.  32,  48--77 (2002)

\bibitem{BPS-sp500}
Bayraktar, E., Poor, H., Sircar, K.: Estimating the fractal dimension of the
  {S\&P}500 index using wavelet analysis. International joirnal of theoretical
  and applied finance  7(5),  615--644 (2004)

\bibitem{blum1997empirical}
Blum, A.: Empirical support for winnow and weighted-majority algorithms:
  Results on a calendar scheduling domain. Machine Learning  26(1),  5--23
  (1997)

\bibitem{blum-mansour}
Blum, A., Mansour, Y.: From external to internal regret. Journal of Machine
  Learning Research pp. 1307--1324 (2007)

\bibitem{cesa1997use}
Cesa-Bianchi, N., Freund, Y., Haussler, D., Helmbold, D., Schapire, R.,
  Warmuth, M.: How to use expert advice. Journal of the ACM (JACM)  44(3),
  427--485 (1997)

\bibitem{plg}
Cesa-Bianchi, N., Lugosi, G.: Prediction, Learning and Games. Cambridge
  University Press (2006)

\bibitem{normalhedge}
Chaudhuri, K., Freund, Y., Hsu, D.: A parameter free hedging algorithm. NIPS
  (2009)

\bibitem{cover-binary}
Cover, T.: Behaviour of sequential predictors of binary sequences. Transactions
  of the Fourth Prague Conference on Information Theory, Statistical Decision
  Functions, Random Processes  (1965)

\bibitem{cover-portfolios}
Cover, T.: Universal portfolios. Mathematical Finance  (1991)

\bibitem{kearns-regret}
Even-Dar, E., Kearns, M., Mansour, Y., Wortman, J.: Regret to the best vs.
  regret to the average. Machine Learning  72,  21--37 (2008)

\bibitem{freund-coin}
Freund, Y.: Predicting a binary sequence almost as well as the optimal biased
  coin. COLT  (1996)

\bibitem{freund-schapire-singer-warmuth}
Freund, Y., Schapire, R.E., Singer, Y., Warmuth., M.K.: Using and combining
  predictors that specialize. STOC pp. 334--343 (1997)

\bibitem{gittins}
Gittins, J.C.: Multi-armed Bandit Allocation Indices. John Wiley (1989)

\bibitem{hazan-seshadhri}
Hazan, E., Seshadhri, C.: Efficient learning algorithms for changing
  environments. ICML pp. 393--400 (2009)

\bibitem{helmbold1998line}
Helmbold, D., Schapire, R., Singer, Y., Warmuth, M.: On-line portfolio
  selection using multiplicative updates. Mathematical Finance  8(4),  325--347
  (1998)

\bibitem{KP11}
Kapralov, M., Panigrahy, R.: Prediction strategies without loss. In: NIPS'11

\bibitem{Kohn-diff}
Kohn, R., Serfaty, S.: A deterministic-control-based approach to fully
  nonlinear parabolic and elliptic equations. Comm. Pure Appl. Math.  63(10),
  1298--1350 (2010)

\bibitem{weighted-majority}
Littlestone, N., Warmuth, M.: The weighted majority algorithm. FOCS  (1989)

\bibitem{mukherjee2008learning}
Mukherjee, I., Schapire, R.: Learning with continuous experts using drifting
  games. In: Algorithmic Learning Theory. pp. 240--255. Springer (2008)

\bibitem{Raic-stein}
Rai\v{c}, M.: Normal approximation with {Stein’s} method. In: Proceedings of
  the Seventh Young Statisticians Meeting (2003)

\bibitem{rakhlin2010online}
Rakhlin, A., Sridharan, K., Tewari, A.: Online learning: Beyond regret. In:
  COLT (2011), also arXiv preprint arXiv:1011.3168

\bibitem{S-fractional}
Sottinen, T.: Fractional brownian motion, random walks and binary market
  models. Finance and Stochastics  5(3),  343--355 (2001)

\bibitem{tsitsiklis}
Tsitsiklis, J.: A short proof of the gittins index theorem. Annals of Applied
  Probability  4 (1994)

\bibitem{vovk-game}
Vovk, V.: A game of prediction with expert advice. Journal of Computer and
  System Sciences  (1998)

\bibitem{vovk}
Vovk, V.: Derandomizing stochastic prediction strategies. Machine Learning pp.
  247--282 (1999)

\end{thebibliography}

\appendix


\section{Prediction Algorithms for Fixed Stopping Time}
\label{apx:fixedTime}

In this section we discuss the optimal regret and the corresponding
betting algorithm for a fixed stopping time $T$, which leads to
strategies that depend on current time $t$ and the stopping time
$T$. We consider the classical non-discounted setting (Theorem
\ref{thm:vanillaRegret}) and the time-discounted setting (Theorem
\ref{thm:discountedFixed}), both with fixed stopping time.

In the non-discounted setting, we show that the optimal regret and
algorithm follow easily from the existing work of 
\cite{cover-binary}.

We note that the resulting prediction algorithms depend on the current
time $t$ and the stopping time $T$.  We will consider the admittedly
more interesting case --- of time-independent strategies --- in the
next section.

\subsection{Non-discounted setting}
\label{sec:nondiscounted}

\cite{cover-binary} gave a precise characterization of
possible payoff curves attainable. First of all, he showed that, if,
for a sequence $\vb\in\{\pm1\}^T$, we denote $g(\vb)$ to be the
payoff/score obtained for sequence $\vb$, then $\sum_\vb
g(\vb)=0$ for all possible algorithms. Cover proves the
following characterization of the curve as a function of the height of
the sequence:
\begin{theorem}[\cite{cover-binary}]
\label{thm:coverCharacterization}
Let $f:\N\to \R$ be the payoff function of an algorithm, where $f(x)$
is the payoff of an algorithm for sequences of height $x$
precisely. Then $f$ is feasible if and only if: 1) $\sum_{x=0}^T {T
  \choose x} f(T-2x)=0$ and 2) $|f(x+1)-f(x)|\le 1$ ($f$ is
Lipschitz).
\end{theorem}

From the above theorem we have the following corollary.
\begin{corollary}
$f(x)=|x|-R$ is feasible for $R=\sqrt{2\over
    \pi}\sqrt{T}+O(1)$, and this is the minimum $R$ for which
  this is feasible.
\end{corollary}
\begin{proof}
Note that Theorem \ref{thm:coverCharacterization} holds for the payoff
function $f(x)=|x|-R$, where $R=\E{\vb\in \{\pm1\}^T} {|\sum_i
  \vb_i|}$. To compute this value $R$, we use following standard
approximation: $\left|\E{\vb\in \{\pm1\}^T} {|\sum_i \vb_i|}-\E{x\sim
  \phi}{|\sqrt{T}x|}\right|\le O(1)$, where $\phi(x)$ is the normal
distribution (see, e.g., \cite{Raic-stein}, Theorem 3.4). Furthermore,
we have that $\E{x\sim \phi} {|\sqrt{T}x|}=\sqrt{2\over
  \pi}\sqrt{T}$. The corollary follows.
\end{proof}

To recover the actual prediction algorithm, we employ the following
standard dynamic programming. Namely, define $s_t(x)$ to be the
minimal necessary algorithm payoff, after $t^{th}$ time step for
height $x$, in order to obtain payoffs of $s_T(y)=f(y)=|y|-R$. In
particular, if $\bt_{t}(x)$ denotes the prediction (bet) at time $t$
assuming the current height is $x$, we have that
$s_{t}(x)=\min_{|b_{t}(x)|\le 1}\max\{s_{t+1}(x+1)-\bt_{t+1}(x),
s_{t+1}(x-1)+\bt_{t+1}(x)\}$. Suppose we ignore the boundedness of
$b_{t}(x)$, then the minimum is achieved for
$b_{t+1}(x)=\tfrac{1}{2}(s_{t+1}(x+1)-s_{t+1}(x-1))$. Note that this
way we obtain $s_{0}(0)=\E{\vb\in \{\pm1\}^T} {f(\sum_i \vb_i)}=0$
(which gives a different proof of the above theorem). But these values
of $\bt$ actually satisfy $|\bt_{t}(x)|\le 1$, since if the Lipschitz
condition holds at time $t$, then it also holds at time $t-1$. Hence
there was no loss of generality of dropping the boundedness of
$\bt_{t}(x)$'s. In particular, we have that $\bt_t(x)=
\tfrac{1}{2^{T-t+1}} \sum_{\vb\in \{\pm1\}^{T-t+1}}\left(|x+1+\sum_j \vb_j|-|x-1+\sum_j \vb_j|\right)$.

This concludes the proof of Theorem~\ref{thm:vanillaRegret} to show
the optimal regret and prediction algorithm for the vanilla fixed
stopping time setting. Note that the prediction algorithm $\bt_t(x)$
depends on the current time: for example, for $t$ close to $T$ the all
bet values are close to 1, whereas for small $t$'s we obtain very
small values of $\bt_t(x)$.

\subsection{Time-discounted setting}
\label{apx:discountedFixed}

We prove the following theorem for the time-discounted setting with
fixed stopping time $T$, by extending the characterization given in
Section \ref{sec:nondiscounted}.

\begin{theorem}
\label{thm:discountedFixed}
Consider the problem of time-discounted prediction of binary sequence for
``window size'' $n$. Fix the discount factor $\rho=1-1/n$.  For any
fixed time $T$, $f_T$ is feasible iff $\int f(x) p_T(x)\dx=0$ where
$p_T(x)$ is the probability of a (decayed) random walk to end at
height $x$ and $f$ is $1$-Lipshitz (for bounded bet value).

There is an algorithm (betting strategy) achieving this optimal regret
and has $f(x) = |x|-R_\rho$, where $ R_\rho=\min_f R_T(f)=
\sqrt{\frac{2}{\pi}} \sqrt{\alpha}+O(1)$, and
$\alpha=\tfrac{1-\rho^{2T}}{1-\rho^2}$. Note that $\alpha\rightarrow
T$ when $T\ll n$, and $\alpha\rightarrow n/2$ when $T\gg n$. The
betting strategy may be computing via dynamic programming.
\end{theorem}

First, we need to count the number of random walks achieving a certain
discounted height $x$. When the height was not discounted, this was
simply a binomial distribution, which we approximated by a normal
distribution. It turns out that, in the discounted height case, the
height distribution is also approaches normal distribution at the
limit. Specifically, we show the following lemma.

\begin{lemma}
Consider the time-discounted setting, with discount $\rho=1-1/n$ for
some $n\ge 1$. Let $p_T(x)$ be the probability that a random binary
sequence of length $T$ has discounted height $x\in[-n,n]$. Then, as
$T$ goes to infinity, the probability distribution of the discounted
height, scaled down by $\sqrt{\alpha}$, converges to the normal
distribution $N(0,1)$, where $\alpha=\tfrac{1-\rho^{2T}}{1-\rho^2}$.

Furthermore, $\E{\vb\in \{\pm1\}^T}{|\sum_{i\ge 1}
  \vb_i\rho^{T-i}|}=\sqrt{2\over \pi}\sqrt{\alpha}\pm O(1)$.
\end{lemma}
\begin{proof}
Note that the height is distributed as $x=\sum_{i=1}^{T} \vb_i
\rho^{T-i}$ where $\vb_i$ are random $\pm1$. Then, by Lyapunov
central limit theorem, we have that
$\tfrac{1}{\sqrt{\alpha}}\sum_{i=1}^{T} \vb_i \rho^{T-i}$ tends to
$N(0,1)$ as long as $T=\omega_n(1)$.

Again, we have that (see, e.g., \cite{Raic-stein}, Theorem 3.4)
$
\left|
\E{\vb\in \{\pm1\}^T}{|\sum_{i\ge 1}
  \vb_i\rho^{T-i}|}
-
\E{x\sim \phi}{|\sqrt{\alpha}\cdot x|}
\right|
=
O(\alpha^{-1})\cdot \sum_{i=0}^{T} \rho^{3i}
=
O(1).
$
Hence, we obtain that  $\E{\vb\in \{\pm1\}^T}{|\sum_{i\ge 1}
  \vb_i\rho^{T-i}|}=\sqrt{2\over \pi}\sqrt{\alpha}\pm O(1)$.
\end{proof}

The rest of the proof of Theorem \ref{thm:discountedFixed} follows
along the same lines of Theorem \ref{thm:vanillaRegret}. Specifically,
one can employ the same dynamic programming (for all possible
discounted heights). We again have that $s_{0}(0)=\E{x\sim p_T}{f(x)}$
for any desired target function $f$. The only way $s_0(0)=0$ is when
$\E{x\sim p_T}{f(x)}=0$. As long as $f$ is also Lipschitz, the dynamic
programming will recover the betting strategy with bounded bets
$|\bt_{t}(x)|\le 1$. As in the previous setting, note that the betting
strategy $\bt_t$ depends on the time $t$: it is small at the
beginning, and gets closer to 1 for large values of $t$ (close to
$T$).

\section{Trade-off with two experts}
\label{sec:tradeoff}

In this section we will prove Theorem~\ref{twoexperts} by proving an equivalence 
between the sequence prediction problem and the two-experts problem. 
In each round of the experts problem, each expert has a payoff in the range
$[0,1]$ that is unknown to the algorithm. For two experts, let
$b_{1,t}, b_{2,t}$ denote the payoffs of the two experts. The algorithm
pulls the each arm (expert) with probability $\bt_{1,t}, \bt_{2,t} \in
[0,1]$ respectively where $\bt_{1,t} + \bt_{2,t} = 1$.  The payoff of
the algorithm is $A = \sum_{t=1}^T b_{1,t}\bt_{1,t} + b_{2,t}\bt_{2,t}
$. Let $X_1 = \sum_{t=1}^T b_{1,t}$ We will study the regret trade-off
$R_1, R_2$ with respect to these two experts which means that $A \ge
X_1 - R_1$ and $A \ge X_2 - R_2$.

For this we we translate it into an instance of the sequence
prediction problem where we show how we can obtain a tradeoff between
regret $R$ and loss $L$, which is defined as the minumum payoff of the
algorithm.  With two experts, the regret/loss tradeoff in the sequence
prediction problem is related to regret trade-off for the two experts
problem.  Let $R$, $L$ be feasible upper bounds on the regret and loss
in the sequence prediction problem in the worst case; Let $R_o, L_o$
be feasible upper bounds on the regret and loss with version of the
sequence prediction problem with one sided bets (that is $\bt_t$
cannot be negative; the feasible payoff curves for this case is a
simple variant of $F_{c_1,c_2}$ where $F'$ is capped to lie in
$[0,1]$.) Let $R_1$, $R_2$ be feasible upper bounds in regret with
respect to expert one and expert two in the worst case.  Another
variant that has been asked before is a tradeoff between regret to the
average and regret to the max (see \cite{kearns-regret,KP11}). Let
$R_m$, $R_a$ be feasible upper bounds on the regret to the max and
regret to the average with two experts in the worst case.

Theorem~\ref{twoexperts} follows from the following two lemmas.

\begin{lemma}
\label{lem:tradeOffReduction}
Regret and loss $R,L$ is feasible in the sequence prediction problem
if and only if $R_m= R/2, R_a= L/2$ is feasible for regret to the max
and regret to the average in the two experts problem.

$R_o,L_o$ is feasible in the sequence prediction problem (with one
sided bets) if and only if $R_1 = L_o, R_2 = R_o$ is feasible for
regret to the first expert and regret to the second expert in the two
experts setting.
\end{lemma}

For $x \ge 0$, let $\T(x) = h(g^{-1}(x))$ where $g(x) = e^{x^2}, h(x)
= \erfi(x)$. Note that $\T(x)=\erfi(\sqrt{\ln x})$.

\begin{proof}[Proof of Lemma \ref{lem:tradeOffReduction}]
First we look at reduction from the regret to the average and regret
to the max problem. We can reduce this problem to our sequence
prediction problem by producing at time $t$, $b_t = (b_{1,t} -
b_{2,t})/2$. A bet $\tilde b_t$ in our sequence prediction problem can be
translated back into probabilities $\bt_{1,t} = (1+\bt_t)/2$ and
$(1-\bt_t)/2$ for the two experts. A payoff $A$ in the original
problem gets translated into payoff $\sum_t b_{1,t} (1+\bt_t)/2 +
b_{2,t} (1-\bt_t)/2 = (X_{1} + X_{2})/2 + A$ in the two experts case.
In this reduction the loss $L$ gets mapped to $R_a$ and the regret $R$
gets mapped to $R_m$. However note that $b_t$ is now in the range
$[0,1/2]$. Therefore we need to scale it by $2$ to reduce it to the
standard version of the original problem. Conversely, given an
sequence $b_t$ of the prediction problem we can convert it into two
experts with payoffs $b_{1,t} = (1+b_t)/2, b_{2,t} = (1-b_t)/2$. The
average expert has payoff $T/2$.  A payoff of $A$ in prediction
problem can be obtained from a sequence of arm pulling probabilities
with payoff $T/2 + A/2$ by interpreting the arm pulling probabilities
as $(1\pm \bt_t)/2$ since $\sum_t \frac{ (1+b_t)}{2}\frac{
  (1+\bt_t)}{2} + \frac{ (1-b_t)}{2}\frac{ (1-\bt_t)}{2} = T/2 + A/2$.

Next we look at regrets $R_1, R_2$ with respect to the two
experts. Given a sequence of payoffs to for the two experts we can
reduce it to a sequence for the (one sided ) prediction problem by
setting $b_t = b_{2,t} - b_{1,t}$. A bet $\tilde b_t$ in the prediction
problem can be translated to probabilities $\bt_{1,t} = 1-\bt_t$ and
$\bt_{2,t} = \bt_t$ for the two experts.  A payoff $A$ in the
prediction problem gets translated into payoff $\sum_t (1-\bt_t)
b_{1,t} + \bt_t b_{2,t} = X_{1} + A$ in the two experts case where a
zero regret in the prediction would correspond to $A = X_2 -
X_1$. Thus a loss of $L_o$ translates to a regret $R_1 = L_o$ with
respect to the first arm. And regret $R_o$ translates to regret $R_2 =
R_o$ with respect to the second arm. Thus if $R_o, L_o$ is feasible
then so is $R_1 = R_o, R_2 = L_o$. Conversely, given an instance of
the prediction problem with one sided bets, we can convert it to a
version of the two armed problem by setting $b_{2,t} = b_t, b_{1,t} = 0$
if $b_t \ge 0$ and $b_{2,t} = 0, b_{1,t} = -b_t$ otherwise.  A bet
$\tilde b_t$ is used in our original problem if the arms are pulled
with probabilities $1-\tilde b_t$ and $\tilde b_t$ respectively. The
payoff in the experts problem is $X_1 + \sum_t \tilde b_t (b_{2,t} -
b_{1,t})$. So regrets $R_1, R_2$ will translate to $L_o = R_1, R_o =
R_2$ in the prediction problem with one sided bets.

The above reduction also works for the time-discounted case.
\end{proof}

\begin{lemma}
\label{lem:tradeOffDer}
Let $R,L,R_0,L_0$ be normalized by a factor $\sqrt{n}$ (scaled down).
$R,L$ is feasible in the original problem if and only if $\T(R/L) =1/
(\sqrt \pi L)$.

$R_o,L_o$ is feasible in the original problem (with one sided bets)
if and only if there is an $\alpha > 0$ so that $\T(\alpha L_o) +
\T(\alpha R_o ) \ge \alpha /\sqrt \pi$.
\end{lemma}


\begin{proof}[Proof of Lemma \ref{lem:tradeOffDer}]
The best tradeoffs for $R, L$ is attained when $F$ is symmetric; that
is, $F = c_1 (x\erfi(x) - e^{x^2}/\sqrt\pi)$ with the slope capped in
the interval $[1,-1]$. Here $L = c_1/\sqrt \pi$ corresponds to the
minimum value attained at $x=0$. $R$ is obtained by looking at $x - F$
at the point $x_0$ where $F' = 1$ giving $c_1 \erfi(x_0) = 1$ implying
$R = x - F = c_1 e^{x_0^2}/\sqrt\pi $. Thus $1/(\sqrt \pi L) =
\erfi(x_0)$ and $R/L =e^{x_0^2}$, implying $\T(R/L) = 1/(L\sqrt \pi)$.

In the case of one sided bets, we look at the curve $F = c_1(x\erfi(x)
- e^{x^2}/\sqrt\pi) + c_2 x$ where additionally the derivative is
capped in the interval $[0,1]$. Loss $L_o$ is maximized at the minimum
point $x_1$ where $F' = 0$ giving $c_1 \erfi(x_1) + c_2 = 0$ implying
$L_o = -F(x_1) = c_1 e^{x^2}/\sqrt\pi$. Regret $R_o$ is maximized at
$x_0$ where $F' =1$ (which means $c_1 \erfi(x_1) + c_2 = 1$) giving
$R_o = x - F = c_1 e^{x_1^2}/\sqrt\pi$. Since $e^{x^2}$ is even and
$\erfi(x)$ is odd, $\T(L_o \sqrt \pi/c_1) = |c_2/c_1|$ and $\T(R_o
\sqrt \pi /c_1) = |(1-c_2)/c_1|$. For a given $c_1 \ge 0$ (as
otherwise regret is infinity), a $c_2$ exists if and only if $\T(L_o
\sqrt \pi /c_1) + \T(R_o \sqrt \pi /c_1) \ge 1/c_1$.
\end{proof}

\section{Multi-scale Optimal Regret}
\label{apx:multiScale}

We now show how the framework can be extended to the multiple time
scales. The sequence $b_t$ may have trends at some unknown time scale
and therefore it is important that the algorithm has small regret not
just at one time scale but simultaneously at many timescales. We will
now prove that (with unbounded bets) there are (normalized) payoff
functions $g_1(x_1,x_2)$ and $g_2(x_1,x_2)$ at time scales $an$ and
$bn$ if and only if it satisfies the conditions in
Theorem~\ref{thm:multitime}.

\begin{proof}[Proof of Theorem \ref{thm:multitime}]
If $\bt(x_1,x_2)$ is the betting function. then as before we get
$ \rho_1 f_1(x_1,x_2) + b\bt(x_1,x_2) \ge f_1(\rho_1 x_1+b, \rho_2 x_2 + b)$ for $b \in \{-1,1\}$
and 
$ \rho_2 f_1(x_1,x_2) + b\bt(x_1,x_2) \ge f_2(\rho_1 x_1+b, \rho_2 x_2 + b)$ for $b \in \{-1,1\}$

Further these conditions are sufficient. 
Simplifying  we get
\begin{eqnarray*}
 \rho_1 f_1(x_1,x_2) + \bt(x_1,x_2) \ge f_1(\rho_1 x_1+1, \rho_2 x_2 + 1)\\
 \rho_1 f_1(x_1,x_2) - \bt(x_1,x_2) \ge f_1(\rho_1 x_1-1, \rho_2 x_2 - 1)
\end{eqnarray*}
This is satisfied if and only if

$$ \rho_1 f_1(x_1,x_2) - (1/2)(f_1(\rho_1 x_1+1, \rho_2 x_2 + 1) +
f_1(\rho_1 x_1-1, \rho_2 x_2 - 1)) \ge
$$
$$ | (1/2) (f_1(\rho_1 x_1+1, \rho_2 x_2 + 1) -  f_1(\rho_1 x_1-1, \rho_2 x_2 - 1)) - \bt(x_1,x_2)|
$$ 

To see this, note that if $ \bt(x_1,x_2) = (1/2) (f_1(\rho_1 x_1+1,
\rho_2 x_2 + 1) - f_1(\rho_1 x_1-1, \rho_2 x_2 - 1))$ then the two
inequalities become identical. Otherwise we can denote the difference
by $\Delta$ and we get that the left hand side has to be $\ge \pm
\Delta$.

Similarly we get $\rho_2 f_2(x_1,x_2) - (1/2)(f_1(\rho_1 x_1+1, \rho_2
x_2 + 1) + f_1(\rho_1 x_1-1, \rho_2 x_2 - 1)) \ge | (1/2) (f_2(\rho_1
x_1+1, \rho_2 x_2 + 1) - f_2(\rho_1 x_1-1, \rho_2 x_2 - 1)) -
\bt(x_1,x_2)|$

We can write these as $L_1 \ge |R_1 - \bt|$ and $L_2 \ge |R_2 - \bt|$.

Note that for such a $\bt$ to exist it is necessary and sufficient
that $L_1 + L_2 \ge |R_1 - R_2|$ and $L_1 \ge 0$ and $L_2 \ge 0$.

Now rescaling into functions $g_1$ and $g_2$ we get
$$
\begin{array}{ll}
L_1 &\\
 =& \rho_1 f_1(x_1,x_2) - (1/2)(f_1(\rho_1 x_1+1, \rho_2 x_2 + 1) +  f_1(\rho_1 x_1-1, \rho_2 x_2 - 1))\\
=& (1- a_1^2 \delta^2) g_1 (x_1,x_2) - \tfrac{1}{2} (g_1((1-a_1^2\delta^2)x_1+\delta, (1-b^2\delta^2) x_2 + \delta) +  g_1((1-a^2\delta^2)x_1-\delta, (1-b^2\delta^2) x_2 - \delta))\\
=& -a_1^2\delta^2  g_1 (x_1,x_2) + \delta^2  (a_1^2\pd{}{x_1}+a_2^2\pd{}{x_2}) g_1 + (1/2)((-a_1^2\delta^2 + \delta)\pd{}{x_1} + ((-a_2^2\delta^2 + \delta)\pd{}{x_2})^2  
\\
&+  (1/2)((-a_1^2\delta^2 - \delta)\pd{}{x_1} + ((-a_2^2\delta^2 - \delta)\pd{}{x_2})^2)
\end{array}
$$

Dividing by $\delta^2$ and taking limit as $\delta \rightarrow 0$ we get
$ -a_1^2 g_1 (x_1,x_2) + (a_1^2 x_1\pd{}{x_1}+a_2^2x_2\pd{}{x_2}) g_1
- (1/2)(\pd{}{x_1}+\pd{}{x_2})^2 g_1$.

Thus we have $E_1 = -a_1^2 g_1 (x_1,x_2) + (a_1^2 x_1\pd{}{x_1}+a_2^2
x_2\pd{}{x_2}) g_1 - (1/2)(\pd{}{x_1}+\pd{}{x_2})^2 g_1 \ge 0$ and
$E_2 = -a_2^2 g_2 (x_1,x_2) + (a_1^2 x_1 \pd{}{x_1}+a_2^2
x_2\pd{}{x_2}) g_2 - (1/2)(\pd{}{x_1}+\pd{}{x_2})^2 g_2 \ge 0$.

Now $R_1 = (1/2) (f_1(\rho_1 x_1+1, \rho_2 x_2 + 1) - f_1(\rho_1
x_1-1, \rho_2 x_2 - 1))$ After scaling this becomes in the limit.  $ =
(1/2) (g_1((1-a_1^2\delta^2)x_1+\delta, (1-b^2\delta^2) x_2 + \delta)
- g_1((1-a^2\delta^2)x_1-\delta, (1-b^2\delta^2) x_2 - \delta)) =
\delta^2 (\pd{}{x_1}+\pd{}{x_2}) g_1$.

Dividing by $\delta^2$ we get:
$E_1 + E_2 \ge | (\pd{}{x_1}+\pd{}{x_2})(g_1 - g_2)|$.
\end{proof}

\end{document}